\documentclass[a4paper]{article}
\usepackage{amsmath,amstext,amsgen,amsbsy,amsopn,amsfonts,amssymb}
\usepackage{easybmat}
\usepackage{graphics}
\usepackage[pdftex]{graphicx}
\usepackage{epsfig}
\usepackage{amsthm}
\usepackage{bm}
\usepackage{algorithm}
\usepackage{algorithmic}
\newtheorem{theorem}{Theorem}

\usepackage{wrapfig}
\usepackage{esvect}
\usepackage{multirow}
\usepackage{array}
\usepackage{xcolor}
\usepackage{hyperref}
\newcommand{\minitab}[2][c]{\begin{tabular}{#1}#2\end{tabular}}
\providecommand{\keywords}[1]{\textbf{\textit{Index terms---}} #1}

\begin{document}
\large
\title{Learning over Inherently Distributed Data}
\author{
Donghui Yan$^{\dag\S}$, Ying Xu$^{\ddag}$
\vspace{0.1in}\\
$^\dag$Department of Mathematics and Program in Data Science\vspace{0.06in}\\
$^{\S}$University of Massachusetts, Dartmouth, MA\vspace{0.1in}\\
$^{\ddag }$Indigo Agroculture Inc, Boston, MA 
\\
}
\date{\today}
\maketitle
\normalsize

\begin{abstract}
\noindent
The recent decades have seen a surge of interests in distributed computing. Existing work focus primarily on either distributed 
computing platforms, data query tools, or, algorithms to divide big data and conquer at 
individual machines etc. It is, however, increasingly often that the data of interest are 
{\it inherently distributed}, i.e., data are stored at multiple distributed sites due to diverse collection channels, business 
operations etc. We propose to enable learning and inference in such a setting via a general 
framework based on the distortion minimizing local transformations. This framework only requires a small 
amount of local signatures to be shared among distributed sites, eliminating the need of having to transmitting big data. 
Computation can be done very efficiently via {\it parallel local computation}. 
The error incurred due to distributed computing vanishes when increasing the size of local signatures. As the shared data need 
not be in their original form, data privacy may also be preserved. Experiments on linear (logistic) regression and Random 
Forests have shown promise of this approach. This framework is expected to apply to a general class of tools in learning and inference with 
the ``continuity" property. 
\end{abstract}



\section{Introduction}
\label{section:introduction}
The last decades has seen a surge of interests in distributed computing. Some driving forces are the prevalence 
of low-cost clustered computers, storage systems and high speed networking \cite{SortingWorkstation1997,FoxGribble1997}, 
which makes it feasible to interconnect many clustered computers, as well as the pressing need of leveraging big data. 
Numerous systems and platforms have been developed. For example, Google's Bigtable system \cite{GoogleFileSys2003,Bigtable2006}, 
the Apache Hadoop/Map-Reduce \cite{MapReduce2004, HDFS2010}, the Spark system \cite{ZahariaChowdhury2010,ZahariaChowdhury2012}, 
and Amzon's AWS cloud etc. The literature is huge, but mostly on distributed system architecture, computing platforms, or data query tools. 
For an overview of recent advances, please refer to  \cite{ChenMaoLiu2014, Dolev2017ASO}.
\\
\\
Our interest is learning and inference. An influential line 
of work is {\it Bag of Little Bootstrap} \cite{BagLittleBootstrap}. It aims at computing a big data version of Bootstrap \cite{Efron1979}. 
The idea is to take many very ``thin" subsamples, then distribute each subsample 
to a computer node, and finally aggregate inferential results from those individual subsamples. In smoothing Spline setting \cite{Wahba1990}, 
\cite{ChengShang2015} explored the tradeoff between computing efficiency and statistical optimality of the 
{\it Divide-and-Conquer} methods in a distributed environment. Also in the Divide-and-Conquer paradigm are a flurry of work, including 
\cite{ChenXie2014} for penalized regression and model selection consistency when the data is too big to fit in the 
memory of a single machine by working on subsamples of the data and then aggregating the resulting models, \cite{VolgushevChaoCheng2019} which carries out inference by learning a quantile function, \cite{ZhangDuchiWainwright2015} studies 
ridge regression, \cite{BatteyFan2015} considers the general distributed estimation and inference, \cite{LeeLiuSunTaylor2017} learns
Lasso-type linear model at individual sites and then aggregate to de-bias, \cite{RichtarikTakac2016} explores coordinate descent for distributed 
data. Additionally, \cite{MaMahoneyYu2015} 
considered linear least-square regressions on big data using a leverage-based sampling scheme, \cite{RosenblattNadler2016} studied 
the optimality of averaging in distributed computing. One common thread among these is that all assume that the algorithm has access to 
the full data before delegating subtasks to individual machines, and the data are distributed mainly for improving computational efficiency or 
to accommodate the lack of sufficient computer memory.  
\\
\\
With the emergence of big data, it is increasingly often that the data of interest are {\it inherently distributed}. By {\it inherent} 
we mean that the data are stored at a number of distributed sites (machines or nodes) as a result of data collection or business 
operations etc. This is to be contrasted with work mentioned above, for which the data are distributed mainly for computational 
efficiency (i.e., they do not have to be distributed if one has infinite computing power) or memory considerations and as such 
those data on individual machines are typically of identical distribution. To better appreciate why it is necessary to consider 
inherently distributed data, we give here a few examples. 
A big retail vendor, e.g., {\it Walmart}, has sales data generated from {\it walmart.com}, 
or its {\it Walmart} stores, or its warehouse chains---{\it Sam's Club} etc; such data from different sales channels are distributed 
as they are owned by different business groups. In the {\it Walmart} case, many e-commerce applications need more data---the 
more the merrier---for information, such as ``customer who buys A also buys B", in order to build a better recommendation system 
or user personalization model. It is highly desirable to combine data from different sales channels. Another example is in the insurance 
industry, for example the {\it Valen Data Consortium} and the {\it National Insurance Crime Bureau}. 
Data consortiums are formed to combine data from individual vendors to solve the problem of insufficient sample size or to 
have larger data sets for more robust analytics or modeling. 
\\
\\
Similar examples can be easily found in healthcare. Let's take the {\it Prostate Cancer DREAM Challenge} as an example. The 
DREAM Challenge was organized by a consortium of 10 institutes, including Dana Farber Cancer Institute, Prostate Cancer 
Foundation, UC San Francisco etc. For more details, please refer to the DREAM Challenge web.
The data was provided by the following four organizations
\begin{center}
{\it AstraZeneca, Celgene, Sanofi, Memorial Sloan Kettering Cancer Center}.
\end{center}
The goal is to improve the prediction of survival or toxicity of treatment for patients with {\it metastatic castrate 
resistant prostate cancer} (mCRPC) by combining data from all four data providers. Similar as the Walmart example, 
the benefit of combining data is immediate---the resulting data is much larger, and is potentially less biased since each 
data provider may have its own patients base, which 
are typically of different distributions, and by combining data one gets {\it a more faithful representation of the entire 
population}. Often data providers are not willing to directly transmit its patients records due to concerns 
about privacy, or potential business competitions. Data sharing was enabled by a trusted third party, {\it Project Data 
Sphere, LLC} in DREAM Challenge. But that would be slow in deployment, and require a lot of efforts coordinating data 
providers, and can be very costly. Our goal is to provide a fast and easy to use algorithmic solution. 
\\
\\
There are several challenges in finding an algorithmic solution for data sharing. One is that the data over distributed 
sites may be too big, e.g., the {\it Walmart} example, or too sensitive to be directly shared among different sites, e.g., 
the healthcare example, thus it is not desirable to transmit the original data (and in a big 
amount) across distributed sites. Additionally, data at distributed sites may have very {\it different distributions}, due to 
different data collection mechanisms or the subpopulations involved. To appreciate the difficulty of the 
problem, let us ponder for a while and ask the question: {\it Given that each distributed site has its own data, 
but are prevented from directly accessing large amount of data at other sites, how could one perform a global learning or inference 
on all the data?} 
\\
\\
Clearly, ensemble type of algorithms, for example those discussed in the beginning of this section, will not work in 
a straightforward fashion. This is because data at individual sites may be of a different distribution. One may 
pool results from different sites using a weighted scheme, 
but that would require distributional information of individual sites, which is often not easy to estimate. The approach 
we take is a general learning and inference framework where a signature of data is computed at each local machine 
which are pooled together for learning and inference. The computing of the signatures uses only local data, and
can be done simultaneously thus makes good use of the existing hardware infrastructure. 
Work closely related to ours include \cite{JordanLeeYang2018}, which computes 
the global likelihood by iterating over (thus 
transmission of) likelihoods computed at local machines. Though a single transmission may be small in size but 
there are potentially many of those thus this requires a close coordination among individual machines. Also along the similar line 
is \cite{FanPCA2018} which computes principal eigenspaces at individual nodes and then aggregates. Our approach is different
in that it does not require an in-depth ``surgery" of the target method, rather it provides a general framework readily applicable to 
any method of a ``continuity" property through a procedure that is simple, computationally efficient, and requires only knowledge 
of the local data.   
\\
\\
Our contributions are as follows. Motivated by real world applications, we study a fairly new class of problems---learning and 
inference over inherently distributed data. 
In such a setting, our proposed framework enables learning and inference as if {\it directly} on the full data; the potential error
incurred vanishes when increasing the size of local signatures. Our framework requires low data transmission, and, as the 
signatures need not be in their original form, privacy is also preserved. With the bulk of computation carried out {\it in parallel 
at where the data are located} (i.e., parallel local computation), our approach naturally achieves the parallel effect of divide-and-conquer while 
eliminating the need of transmitting big data. 
Our framework is easy to implement, and can be readily applied to any methods with the continuity property.
\\
\\
The remaining of this paper is organized as follows. In Section~\ref{section:framework}, we will describe our framework 
and explain its implementation. This is followed by a discussion in Section~\ref{section:dmlTheory} on some theory that 
supports our learning framework. In Section~\ref{section:vignetteTools}, we discuss a vignette of three learning tools, 
including linear regression, $L_1$ logistic regression \cite{Tibshirani1996, ParkHastie2007, glmnet2010}, and Random 
Forests (RF, \cite{RF}), under our framework. In Section~\ref{section:experiments}, we present experimental results on 
these tools in a distributed setting. Finally we conclude in Section~\ref{section:conclusions}. 
%
\section{A framework for distributed learning and inference}
\label{section:framework}
Underlying our approach is the notion of {\it continuity}---similar data points would yield similar results for learning and 
inference, thus only one or a few among those similar data points are necessary to keep. The notion of continuity can be 
implemented through a class of data transformations called {\it  distortion minimizing local (DML)} transforms 
\cite{YanHuangJordan2009}. The idea is to represent the data by a small set of representative points (or codewords); 
one can think of this as a small-loss data compression or the representative set as a sketch (or signature) of the full data. 
Since the representative set resembles the full data, properties or estimations on the representative set are expected 
to be close to those on the full set. \cite{YanHuangJordan2009} applies DML transforms to spectral clustering \cite{ShiMalik2000,NJW2002}
for fast computation and shows that the error incurred due to DML vanishes when increasing the representative set. 
Similar ideas were explored in \cite{coresets,Chen2011LargeSS}, and have since been applied to 
a number of computation-intensive algorithms, assuming the full data are already in one place and can fit in a single machine. 
In all those previous work, DMLs were mainly introduced to address the computational challenge. 
\\
\\
When the data are inherently distributed, will the notion of continuity and the related DML methodology be still applicable? 
This is a crucial and challenging question. Our theoretical analysis gives an affirmative answer to this, due to the {\it local} 
nature of DML. That is, DML can be done locally, without having to see the full data. Thus DML can be applied to data at 
each distributed site separately. If one can gather local signatures from all individual sites, then approximate learning and 
inference can be easily carried out. Thus, as long as the local data transforms are fine enough, a large class of learning 
algorithms will be able to yield result {\it as good as directly using the full data}. 
\begin{figure}[h]
\centering
\begin{center}
\hspace{0cm}
\includegraphics[scale=0.26,clip]{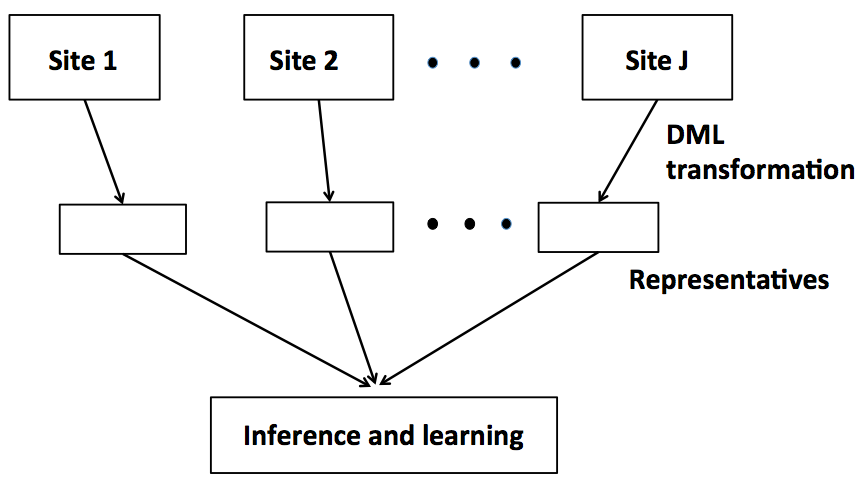}
\end{center}
\caption{\it Learning and inference over inherently distributed data. } 
\label{figure:distInfArch}
\end{figure}
\\
\\
Our framework for learning and inference over inherently distributed data is easy to implement. It consists of three steps:
 \begin{itemize}
 \item[1)] Apply DML to data at each distributed site
 \item[2)] Collect representative points from all distributed sites, and carry out learning and inference on the combined representative set
 \item[3)] Populate the learned model or results to all involved distributed sites.
 \end{itemize}
Figure~\ref{figure:distInfArch} is an illustration of the framework. Clearly algorithms designed under such a 
framework would eliminate the need of having to transmit large amount of data among distributed sites---only those representative 
points need to be transmitted. If the representative set is computed as the average of similar points, then data privacy can be 
preserved since then none of the original data points are ever transmitted. Additionally, learning and inference is only performed on 
the representative set and DML can be done in parallel at individual nodes, the overall computation can 
speed up as well. For the rest of this section, we will describe two implementations of DML proposed in \cite{YanHuangJordan2009}, 
one based on $K$-means clustering \cite{hartiganWong1979, Pollard1981, lloyd1982}, and the other 
based on space partitioning via K-D tree \cite{Bentley1975} or its variants, random projection trees (rpTrees) \cite{DasguptaFreund2008, rpForests2019}. 
\subsection{Implementation of DML transforms}
In general, the DML transforms are required to be computationally efficient while incurring very little loss in information.
It is known that, empirically, both K-means clustering (under a typical implementation such as the Lloyd's algorithm \cite{lloyd1982}) 
and K-D tree related algorithms scale linearly (or with an additional $\log$ factor) with the number of data points, thus are suitable for big data. 
Later in Section~\ref{section:dmlTheory}, we will show both incur little loss in accuracy for pattern classification problems. 
\\
\\
When DML is implemented by $K$-means clustering, each distributed site performs $K$-means clustering separately. Note that, 
here the number of clusters, $K$, may be different for different sites; the only requirement is that $K$ is 
reasonably large but meanwhile small compared to the number of data points at each node. The set of representative 
points are taken as the cluster centers, or averaged, by a metric appropriate for the underlying data, over data points within 
the same cluster.    
\begin{figure}[h]
\centering
\begin{center}
\hspace{0cm}
\includegraphics[scale=0.26,clip]{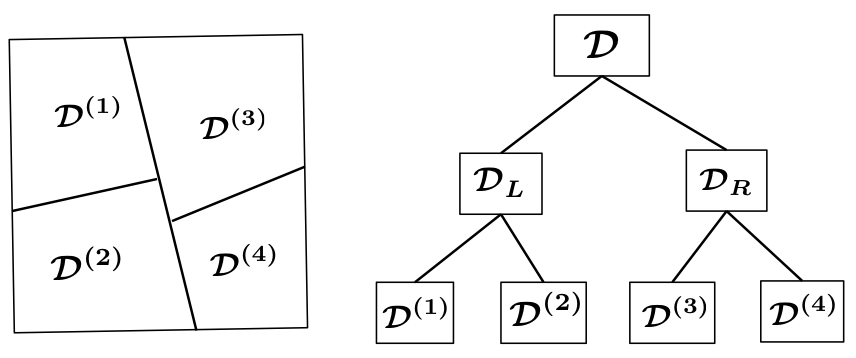}
\end{center}
\caption{\it Illustration of space partition and random projection trees. 
One starts with the root node, $\mathcal{D}$, which corresponds to the full data. After the first split, $\mathcal{D}$ 
is partitioned into its two child nodes, $\{\mathcal{D}_{L}, \mathcal{D}_{R}\}$. The second split partitions $\mathcal{D}_{L}$ 
into $\mathcal{D}_{L}=\mathcal{D}^{(1)} \cup \mathcal{D}^{(2)}$. The third split 
divides $\mathcal{D}_{R}$ into $\mathcal{D}_{R}=\mathcal{D}^{(3)} \cup \mathcal{D}^{(4)}$. 
This process continues until a stopping criterion is met.} 
\label{figure:rpTree}
\end{figure}
\\
To see how recursive space partition and the K-D tree can be used for a DML transform, we first describe the growth of a K-D tree. 
Let the collection of all data at a distributed site form the root node, $\mathcal{D}$, of a tree. Now one variable, say $v$ (index of a 
variable), is selected to split the root node into two child nodes, $\mathcal{D}_L$ and $\mathcal{D}_R$, according to whether a data 
points has its $v$-th coordinate smaller or larger than a cutoff value. For each of $\mathcal{D}_L$ and $\mathcal{D}_R$, a similar 
recursive procedure is followed. That is, pick a variable and then split the node by comparing each data point in the node to a cutoff 
value. This process continues until a stopping criterion is met. In this work, a node stops growing when its size ({\it the number of data 
points it contains}) is less than a predefined value. Figure~\ref{figure:rpTree} illustrates the growth of a K-D tree. 
\\
\\
As the tree growth can be viewed as generating a space partition of the data, data points in the same leaf node would be 
similar and their average can be used as the representative. If the data has a very high dimension, we can use rpTrees \cite{DasguptaFreund2008} 
which adapts to the geometry of the underlying data and readily overcome the {\it curse of dimensionality}. 
In this work, the rpTrees implementation is 
adopted from \cite{rpForests2019}, and for simplicity we split the tree nodes along the median.
\section{A little theory of DML transforms}
\label{section:dmlTheory}
The DML transforms cause distortion to the data, i.e., the data $(X,Y)$ become $\left(T(X),Y\right)$. The idea of 
the DML-based learning framework is that, if the distortion is small then learning and inference over $\left(T(X),Y\right)$ will differ 
very little from that on $(X,Y)$. Such a property is called {\it continuity}. We expect that the continuity would hold 
for many classes of problems in learning and inference. In particular, this can be shown for the pattern classification problem. 
\\
\\
The DML transforms were initially proposed for a framework of fast approximate spectral clustering algorithms \cite{YanHuangJordan2009}. 
It was shown that {\it continuity} holds for DMLs by both quantization (K-means clustering) and rpTrees. Here we extend the idea to pattern 
classification in a distributed setting. Our results hinge on an important 
theorem established by \cite{FaragoGyorfi1975}, which also allows us to precisely characterize the notion 
of {\it distortion minimizing}. 
\begin{theorem}[\cite{FaragoGyorfi1975}] 
\label{thm:faragoGyorfi1975}
Assume that for a sequence of transformations $T_m, m=1, 2, ...$
\begin{equation*}
|| T_m(X) - X || \rightarrow 0
\end{equation*} 
in probability, where $|| \cdot ||$ denotes the Euclidean norm in $\mathbb{R}^d$. Then, if $L^*$ is the Bayes 
error for $(X,Y)$ and $L_{T_m}^*$ is the Bayes error for $(T_m(X), Y)$, 
\begin{equation*}
L_{T_m}^* \rightarrow L^*.
\end{equation*} 
\end{theorem}
\noindent
Let $T_m$ denote the sequence of DML transforms. Note that here $m$ may depend on the training sample size $n$. 
One desirable property of a particular implementation of DML is the Bayes consistency 
of the classifier fitted on transformed data as if on the original data as $m \rightarrow \infty$. That is,  
\begin{equation*}
L_{T_m} \rightarrow L^* ~~\mbox{as}~ m, n\rightarrow \infty. 
\end{equation*} 
This can be seen as follows. Under the statistical learning framework, we have risk decomposition
\begin{eqnarray*}
L_{T_m} - L^* &=& \left( L_{T_m} - L_{T_m}^* \right) + \left( L_{T_m}^*-L^* \right).
\end{eqnarray*}
If the underlying classifier is universally consistent \cite{DevroyeGL1996}, then
\begin{equation*}
L_{T_m} \rightarrow L_{T_m}^* ~~\mbox{as}~ m, n \rightarrow \infty. 
\end{equation*} 
Assuming that the classifier is universally consistent, then by Theorem~\ref{thm:faragoGyorfi1975}, 
we remain to show that $|| T_m(X) - X || \rightarrow 0$ in probability. Indeed such a property is what 
we really meant by distortion minimizing for the sequence of transformations, $T_m$, and here we
make it concrete. 
\\
\\
\textbf{Definition.} Let $T_m(X)$ be a sequence of data transformations indexed by $m$. $T_m$ is {\it distortion
minimizing} if $|| T_m(X) - X || \rightarrow 0$ in probability as $m \rightarrow \infty$, where $|| \cdot ||$ denotes 
the Euclidean norm in $\mathbb{R}^d$.
\\
\\
Next we will show that the two implementations we consider, that is, data grouping 
by quantization (K-means clustering) and rpTrees, both have the distortion minimizing property. 
\subsection{DML by K-means clustering}
K-means clustering could be understood by vector quantization~\cite{QuantError,Quantization} for which
there is a fairly rich body of literature that characterizes the resulting distortion (called {\it quantization error}). 
We will quote a theorem on quantization error after introducing a few definitions and notations. Then immediately 
we see the distortion minimizing property of K-means clustering as a way of transforming the data by assembling 
some known results.
\\
\\
Let a {\it quantizer} $q$ be defined as $q: \mathbb{R}^d \mapsto
\{y_1,\ldots,y_k\}$ for $y_i \in \mathbb{R}^d$. In terminology of quantization, $k$ is the number of codewords 
(cluster centroids). Quantization can be viewed as a data transformation, $X \mapsto q(X)$. For $X$ generated 
from a random source in $\mathbb{R}^d$, the {\it distortion} of quantizer $q$ is defined as: 
\begin{equation*}
\mathcal{D}(q)=\mathbb{E}||X-q(X)||^{\alpha} ~~\mbox{for some constant}~ \alpha>0.
\end{equation*}
This is the mean square error when $\alpha=2$.  Call $R(q)=\log_2 k$ the {\it rate} of
the quantization code.
Define the {\it distortion-rate function} $\delta(R)$ as
\begin{equation*}
\delta(R)=\inf_{q:~R(q) \leq R} \mathcal{D}(q).
\end{equation*}
$\delta(R)$ is the minimal distortion achievable by a quantizer with at most $2^R$ code words. 
$\delta(R)$ can be characterized in terms of the density function
$f(\cdot)$ of $X$. 
\begin{theorem}[\cite{QuantError,Quantization}]
\label{quantTheorem1} Let the data source $X \in \mathbb{R}^d$ have a probability density $f(x)$. 
Then, for large rates $R$, the distortion-rate function of fixed-rate quantization has the following form:
\begin{equation*}
\delta (R) \cong b_{\alpha,d}\cdot||f||_{d/(d+\alpha)} \cdot k^{-\alpha/d},
\end{equation*}
where $\cong$ means the ratio of the two quantities tends to 1,
$b_{\alpha,d}$ is a constant depending on $\alpha$ and $d$, and
\begin{equation*}
||f||_{d/(d+\alpha)}=\left(\int f^{d/(d+\alpha)}(x)dx\right)^{(d+\alpha)/d}.
\end{equation*}
\end{theorem}
\noindent
A consequence of Theorem~\ref{quantTheorem1} is the distortion minimizing property 
of K-means clustering as a way of transforming the data.
\begin{theorem}
\label{dmlKmeans} Let data source $X \in \mathbb{R}^d$ have a probability density $f(x)$. Let $\{z_1,...,z_k\}$
be the population cluster centroids of K-means clustering on $X$. Let $T_k(X)$ be a transformation that maps $X$ 
to its closest cluster centroid. Then
\begin{equation*}
|| T_k(X) - X || \rightarrow 0 
\end{equation*}
in probability as $k \rightarrow \infty$, where $|| \cdot ||$ denote the Euclidean distance in $\mathbb{R}^d$.
\end{theorem}
\begin{proof}
Please see Appendix. 
\end{proof}
\noindent
\textbf{Remarks}. What Theorem~\ref{dmlKmeans} has established is the shrink in distance of the data to it {\it population}
cluster centroids. For empirical data, the cluster centroids 
found by K-means clustering converge {\it almost surely} to the population ones by the strong consistency of K-means 
clustering \cite{Pollard1981}. Thus the distance between the data and the respective empirical K-means 
cluster centroids shrinks in probability.
\\
\\
A distributed version of Theorem~\ref{dmlKmeans} can also be established which we state as follows.
\begin{theorem}
\label{dmlD} Let data source $X \in \mathbb{R}^d$ have a probability density $f_1, f_2, ..., f_J$, with
a proportion of $p_1, p_2,...,p_J$, respectively, at the $J$ distributed sites $S_1, S_2, ..., S_J$ where $p_1+p_2+...+p_J=1$. Let  
data transformation $X \mapsto T(X)$ be formed by concatenating $J$ distortion minimizing transforms, 
$T^{(j)}$ at site $S_j$ for $j=1,...,J$. That is, 
\begin{equation*}
T(X)=T^{(j)}(X) ~\mbox{if}~ X \in S_j, j=1,2,...,J.
\end{equation*}
Then, $T$ is distortion 
minimizing at $\cup_{j=1}^J S_j$.
\end{theorem}
\begin{proof}
Please see Appendix.
\end{proof}
\subsection{DML by space-partitioning trees}
Data transform by space-partitioning trees works by first recursively partitioning the space in a tree fashion, and 
then map data points within the same tree leaf node to the node centroid. To 
establish the distortion minimizing property of DML by space-partitioning trees, by our previous argument and 
an empirical version of Theorem~\ref{thm:faragoGyorfi1975} (c.f. Problem~32.5 in \cite{DevroyeGL1996}), 
it is enough to show that the diameter of the tree leaf node vanishes. By Theorem~\ref{dmlD}, we are only required 
to show the local (non-distributed) version. We will show this for a simple case, space partition by k-d 
trees that cuts along the median (that is, values smaller than the median become the left child node and larger the right 
child node). The similar idea, along with some technical work, shall extend to the more general case of cutting by a
point selected uniformly at random along the random projection, and we leave that to future work.
\begin{theorem}
\label{dmlKD-tree} Let data source $X \in \mathbb{R}^d$ have a probability density $f(x)$. 
Let $T(X)$ be a transformation that maps $X$ to the centroid of data points in the same
tree leaf node of a k-d tree that cuts by median. A node in a k-d tree stops growing once it has 
less than $K_n=n/(2^k)$ points. Then
\begin{equation*}
|| T(X) - X || \rightarrow 0 
\end{equation*}
in probability as $n \rightarrow \infty$ and $n/(k2^k) \rightarrow \infty$ as $k \rightarrow \infty$, 
where $|| \cdot ||$ denote the Euclidean distance in $\mathbb{R}^d$.
\end{theorem}
\begin{proof}
Please see Appendix.
\end{proof}
\section{A vignette of learning and inference tools in distributed setting}
\label{section:vignetteTools}
To provide a vignette of tools for learning and inference in a distributed setting, we explore the 
performance of three popular methods, namely, linear regression, logistic regression, and RF. Linear regression 
\cite{Rice1995} is one of the most fundamental methods and the simplest modeling tools in 
statistics. Logistic regression \cite{McCullaghNelder1989} is commonly viewed as one of the most popular models 
used in the industry, and $L_1$ regularization \cite{DonohoJohnstone1994,Tibshirani1996,ParkHastie2007, glmnet2010} 
has emerged as a computationally effective approach that achieves a balance of prediction performance and model 
selection. RF is widely acknowledged as one of the most powerful tools in statistics and machine learning \cite{Caruana2006,caruanaKY2008, HTF2009, TACOMA, deepTacoma}. 
\\
\\
For the rest of this section, we will briefly describe linear regression, $L_1$ logistic regression, and RF to make this writing self-contained.
\subsection{Linear regression}
\label{section:vignetteLm}
Linear regression is a fundamental  problem in statistics that aims to uncover a linear relationship between a response variable and some explanatory variables
\begin{equation*}
Y=X \beta +\epsilon
\end{equation*}
with noise term $\epsilon$ such that $\mathbb{E}\epsilon=0$ and $Var(\epsilon)=\sigma^2\bm{I}$. Given a data sample, $(X_i, Y_i), i=1,2,...,n$, the coefficients $\beta$ can be estimated by solving 
\begin{equation*}
\arg\min_{\beta}(Y - X\beta)^T (Y- X\beta),
\end{equation*}
assuming the model fitting is by least square error. The solution is given by 
\begin{equation*}
\hat{\beta} = (X^TX)^{-1}X^TY.
\end{equation*}
Next we will briefly describe how linear regression may work for distributed data. The general framework is illustrated 
in Figure~\ref{figure:distInfArch}, and here we will give more details. 
\\
\\
Assume that there are $J$ distributed sites. Let $\left(\tilde{X}_i^{(s)}, \tilde{Y}_i^{(s)}\right), i=1, 2, ..., N_s$ be the group 
(i.e., cluster in K-means clustering, or tree leaf node in rpTrees) centroids at site $s$ for $s=1,2,...,J$. These group centroids will
be used as the representative points. We will perform a {\it weighted} linear regression on the set of all representative points
\begin{equation*} 
\mathcal{D}_r = \bigcup_{s=1}^J \left\{ \left(\tilde{X}_i^{(s)}, \tilde{Y}_i^{(s)} \right):  i=1, 2, ..., N_s \right\},
\end{equation*}
where $N_s$ are the number of groups at site $s$ for $s=1, 2, ..., J$. Here the weight for each point in $\mathcal{D}_r$ 
is taken as the number of points in the same group. Such a choice of weighting allows to minimize the $L^2$ error (i.e., 
quantization error) between the full data and the representative set. An algorithmic description is given as Algorithm~\ref{algorithm:distLM}.\\
\begin{algorithm}
\caption{~~Linear regression for distributed data} 
\label{algorithm:distLM}
\begin{algorithmic}[1]
\STATE $\mathcal{D}_r \gets \emptyset$, $\mathcal{W}_r \gets \emptyset$;
\FOR {Each $i$ in $\{1,...,S\}$} 
\STATE Apply DML on data at site s; 
\STATE Let $(\tilde{X}_i^{(s)}, \tilde{Y}_i^{(s)}), i=1, 2, ..., N_s$ be group centroids at site $s$;  
\STATE Let $W_i^{(s)}, i=1, 2, ..., N_s$ be the group sizes; 
\STATE $\mathcal{D}_r \gets \mathcal{D}_r \cup \{ (\tilde{X}_i^{(s)}, \tilde{Y}_i^{(s)} ):  i=1, 2, ..., N_s \}$;
\STATE $\mathcal{W}_r \gets \mathcal{W}_r \cup \{ W_i^{(s)}:  i=1, 2, ..., N_s \}$;
\ENDFOR 
\STATE Perform weighted linear regression on $\mathcal{D}_r$ with $\mathcal{W}_r$ as weight;
\end{algorithmic}
\end{algorithm} 
\subsection{Logistic regression}
\label{section:vignetteGlm}
Logistic regression formulates the log odds ratio of the posterior 
probability as a linear model of covariates
\begin{equation*}
\log \frac{P(Y=1 |X=x)}{P(Y=0 |X=x)} = X\beta
\end{equation*}
where $Y \in \{0,1\}$ are called labels, and $X$ are covariates. The coefficients $\beta$ can be estimated by a data 
sample, $(X_i, Y_i), i=1,2,...,n$, with maximum likelihood estimation, or the more modern gradient (or coordinate) 
descent types of algorithms (the {\it glmnet} package \cite{glmnet2010} is used as the solver for this work). When 
the number of variables is large, often a regularization is introduced. We consider $L_1$ regularization here, 
given its popularity in modern high dimensional statistical models \cite{BuhlmannVandeGeer2011}.
\\
\\
As the overall framework is almost the same (except a {\it glmnet} is used in place of regression, where weights of 
individual data points are incorporated via the {\it weights} parameter of the {\it glmnet}) as that for the linear regression 
case, we omit it here. 
\subsection{Random Forests}
\label{section:vignetteRF}
RF \cite{RF} is an ensemble of decision trees with each tree constructed on a bootstrap sample of the data. 
Each tree is built by recursively partitioning the data. At each node (the root node corresponds to the bootstrap 
sample), RF randomly samples a number of features (or sets of features) and then select 
one that would lead to an ``optimal" partition of that node. This process is continued recursively until the tree is 
fully grown, that is, only one data point is left at each leaf node or all pints in the node are of the same class label. 
The superior empirical performance of RF has been demonstrated by a 
number of studies \cite{RF,Caruana2006, caruanaKY2008,HolmesKapelner2009,TACOMA, deepTacoma}. RF is easy to use 
(e.g., very few tuning parameters), has a remarkable built-in ability for feature selection. We use the R package  
{\it randomForest} in this work. The algorithmic description in distributed setting is also similar to that in linear 
regression, and omitted here.
\section{Experiments}
\label{section:experiments}
In this section, we report experimental results. This includes simulations on linear regression, 
$L_1$ logistic regression and RF. We will compare the performance on distributed data to that when the full
data are in one place (called {\it non-distributed} setting; note that no DML is applied). We start by simulations 
on a toy example in Section~\ref{section:expToy}.
Simulations on linear regression, $L_1$ logistic regression and RF are described in Section~\ref{section:expLm}, 
and Section~\ref{section:expL1LogitRF}, respectively. We also carry out simulations to 
compare the performance on distributed data and the average performance of distributed sites assuming each 
site uses only its own data. This is described in Section~\ref{section:furtherExp}. To assess the feasibility of 
our proposed approach to real world data, we conduct experiments on a number of datasets taken from the UC
Irvine Machine Learning Repository \cite{UCI} in Section~\ref{section:expUCI}.
\subsection{An illustrating toy example}
\label{section:expToy}
We start with a toy example. This serves the purpose of helping readers better appreciate the conventional wisdom 
that more data would typically lead to better statistical inference. We consider a simple linear model
\begin{equation*}
y=20+2x+\epsilon,
\end{equation*}
where to exaggerate the bias effect of using partial data, we assume $\epsilon \sim \mathcal{N}(0,50)$. Suppose 
the data are stored over three distributed sites according to the range of the value of the independent variable $x$ by 
$(0, 20)$, $[20,80)$, and $[80,\infty)$, respectively. One can imagine that variable $x$ indicates the age of people 
in some population, and variable $y$ is some health status index. Moreover, the data are split into three distributed 
sites because people of different ages visit specific medical organizations. Here for illustration purpose, we assume 
the response $y$ depends on only one variable, $x$. 
\begin{figure}[!htb]
\centering
\begin{center}
\hspace{-10pt}
\includegraphics[scale=0.42,clip, angle=0]{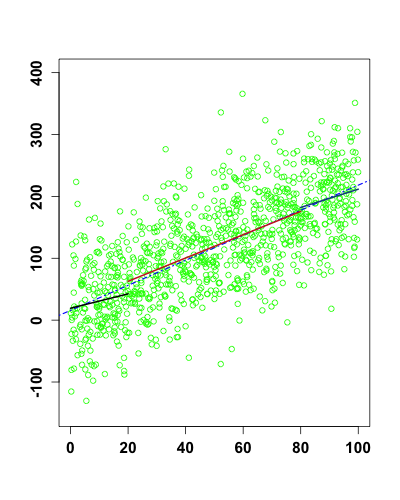}
\includegraphics[scale=0.42,clip, angle=0]{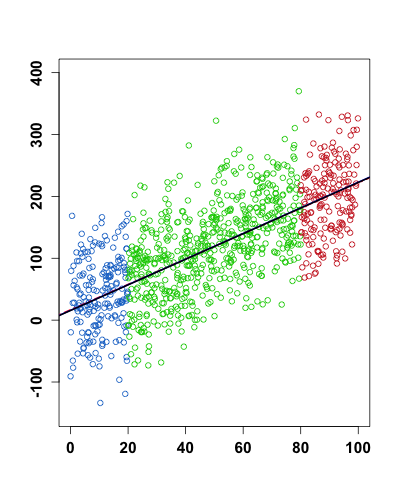}
\end{center}
\abovecaptionskip=-10pt
\caption{\it Linear regression for the toy example. The dot dashed line is the regression line on the full data, while all 
other line segments indicate regression line at individual sites. The right panel shows the lines fitted over the full
data, the fitted regression lines with either K-means clustering or rpTrees; the three completely overlap.} 
\label{figure:lmIllus}
\end{figure}
The left panel in Figure~\ref{figure:lmIllus} is scatter plot of 
the data points with regression lines produced on data from individual distributed sites; the dashed line is the regression 
line in non-distributed setting. There is a noticeable departure of the regression lines fitted by data from individual sites
from the dashed line, even though all the data are drawn from the same generating model (with different support though). 
The right panel in Figure~\ref{figure:lmIllus} shows the regression lines fitted with DMLs, K-means clustering 
or rpTrees, in a distributed setting. It can seen that these lines completely overlap with the regression line in non-distributed 
setting; all agree well with the generating model. This suggests that combining data from distributed sites is desirable, 
and is potentially feasible with our proposed framework.
\subsection{Linear regression}
\label{section:expLm}
More experiments are carried out for multivariate linear regression. The independent variables are generated 
from multivariate Gaussians in $\mathbb{R}^6$ with mean $\mu=(2,0,...,0)$ and covariance matrix 
\begin{equation*}
\Sigma_{i,j}=\rho^{|i-j|}, i,j=1,2,...,6 ~~\mbox{for}~~ \rho=0.1, 0.3, 0.6.
\end{equation*}
The dependent variable is generated according to the following linear model
\begin{equation*}
y = 2 + 3 \cdot x_1 +0.5 \cdot x_3 + 1.2 \cdot x_4 + x_6 +\epsilon
\end{equation*}
with $\epsilon \sim \mathcal{N}(0,1)$.
\\
\\
We consider two distributed sites under two different settings
\begin{itemize}
\item[I.] Site 1 has data with the first independent variable $x_1$ having support $(-\infty, 2)$, and Site 2 
with support  $[2,\infty)$
\item[II.] Site 1 has a random selection of $40\%$ of the data, and Site 2 the rest.
\end{itemize}
To evaluate the fitting, we define the mean squared error of fitted coefficients relative to the true coefficients, 
$\beta=(\beta_0,\beta_1,...,\beta_d)$
\begin{equation*}
MSE=\frac{1}{T}\sum_{l=1}^T \sum_{j=0}^d (\hat{\beta}_j -\beta_j)^2,
\end{equation*}
where $T$ is the number of runs of the simulations, and $j=0$ corresponds to the intercept.
\begin{table}[htp]
\begin{center}
\resizebox{0.9\textwidth}{!}{%
\begin{tabular}{c|cc|cc|cc}
\hline
   & \multicolumn{2}{|c|}{$\bm{\rho=0.1}$}   	& \multicolumn{2}{|c|}{$\bm{\rho=0.3}$}  & \multicolumn{2}{|c}{$\bm{\rho=0.6}$} \\
    \hline
$\beta_0$  &2.0001 &2.0000      &1.9997 &2.0001      &1.9997 &2.0003 \\
                  &2.0000  &2.0005     &1.9994 &2.0001      &2.0009 &2.0003 \\
                  &1.9982  &1.9998     &1.9985 &2.0008      &2.0018 &1.9985 \\
$\beta_1$  &2.9998 &2.9996      &3.0007 &2.9999       &3.0000 &2.9999\\
                  &2.9998 &2.9998      &3.0008 &2.9998       &2.9994 &3.0000\\
                  &3.0014 &3.0001      &3.0012 &2.9994       &2.9990 &3.0007\\
$\beta_2$  &0.0002  &-0.0004    &-0.0001 &0.0005      &-0.0005 &0.0009\\
                  &0.0005  &-0.0006    &-0.0005 &0.0005      &-0.0002 &0.0012\\
                  &-0.0004 &0.0003     &0.0006  &0.0000      &0.0006  &-0.0004\\
$\beta_3$  &0.4997  &0.5004     &0.5000 &0.4989       &0.5000 &0.4990 \\
                  &0.4993  &0.4997     &0.5002  &0.4987       &0.5000 &0.4986\\    
                  &0.4998  &0.5005     &0.4987  &0.4998       &0.4979 &0.4994\\ 
$\beta_4$  &1.1996  &1.2000     &1.2003 &1.2000        &1.2002  &1.1993\\
                  &1.1994  &1.2002     &1.2000 &1.2001        &1.2006   &1.1995\\
                  &1.2000  &1.2002     &1.2011  &1.2000        &1.2015   &1.1995\\
$\beta_5$  &-0.0003 &0.0004     &-0.0007 &0.0001       &0.0008  &-0.0004\\
                  &0.0000  &0.0001     &-0.0006 &0.0003        &0.0004 &-0.0005\\
                  &-0.0006 &-0.0005    &-0.0009 &0.0004        &0.0001  &-0.0008\\
$\beta_6$  &0.9993 &1.0003      &0.9998 &1.0000         &1.0002  &1.0004\\
                  &0.9992  &1.0005     &0.9995 &0.9997          &1.0002  &1.0004\\
                  &0.9993  &0.9998     &1.0008  &1.0005         &0.9999  &0.9998\\
\hline
MSE           &2.9163  &2.9291    &3.1175 &3.0371          &4.9654 &5.0052\\
($10^{-4}$)  &3.2749  &3.5739   &3.7341 &3.8144          &6.0830 &7.0598\\
                    &3.4861  &4.0286   &3.9084 &4.0684          &7.3849 &7.2617\\
\hline
\end{tabular}}
\end{center}
\caption{\it Fitted coefficients of linear model. In each cell, there are two columns--the first column for setting I 
and the second for setting II. In each cell, there are three rows; the first is for non-distributed, 
the second is for distributed under K-means clustering, and the third for distributed under rpTrees.  
} \label{table:lm}
\end{table}
\normalsize
\\
\\
A total of 40000 data points are generated, and the results are averaged over 100 runs. Table~\ref{table:lm} 
shows the fitted regression coefficients of the linear model in non-distributed setting, and that under K-means clustering 
and rpTrees. The data compression ratio for K-means clustering is 40:1, and approximately 40:1 for rpTrees (the 
maximum tree node size is 40). We see that, in all cases, the discrepancy in the regression coefficients for 
distributed and non-distributed data is small. The MSE is also small when comparing fitted coefficients to the true
ones.

\subsection{$L_1$ logistic regression and RF}
\label{section:expL1LogitRF}
We conduct experiments with $X$ being generated by a 4-component Gaussian mixture specified as 
\begin{equation*}
\frac{1}{4}\sum_{i=1}^4 \mathcal{N}(\mu_i, \Sigma) 
\end{equation*}
with the covariance matrix $\Sigma$ defined by 
\begin{equation*}
\Sigma_{i,j}=\rho^{\vert i-j \vert}, ~~\mbox{for}~ \rho=0.1, 0.3, 0.6.
\end{equation*}
The mixture centers are generated as follows. All components of $\mu_1 \in \mathbb{R}^{100}$ are generated uniformly
at random from the interval $[0,0.4]$. $\mu_2$ has the $1^{st}$ half of its components  (i.e., the first 50) the 
same as $\mu_1$, while the other half negative of those of $\mu_1$, $\mu_3=-\mu_2$ and $\mu_4=-\mu_1$.
The mixture component ID of each data point is used as the class label $Y$. We use the classification accuracy on the test 
set as the performance metric.
\begin{figure}[h]
\centering
\begin{center}
\hspace{-0.5cm}
\includegraphics[scale=0.54,clip]{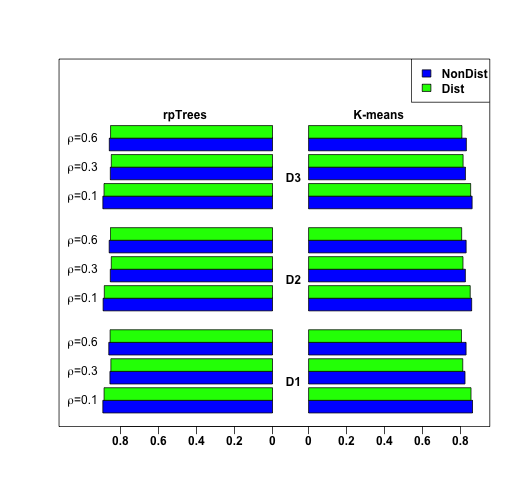}
\end{center}
\abovecaptionskip=-20pt
\caption{\it Test set accuracy of $L_1$ logistic regression under simulation setting $D_1, D_2, D_3$, and different $\rho$'s. 
{\it `NonDist'} means all the data are in one place and no DML is applied, and {\it `Dist'} indicates our proposed distributed algorithm.  } 
\label{figure:logits}
\end{figure}
For simplicity, we only consider the case where there are two sites in a distributed environment. It should be easy
to extend to more distributed sites. Let $\mathcal{C}_i, i=1, 2, 3, 4$, denote the data from the four mixture components. 
To simulate how the data may be partitioned in a distributed environment, we create the following three scenarios:
\begin{itemize}
\item[$D_1$:] Site 1 has $\mathcal{C}_1 + \mathcal{C}_2$, and Site 2 has those from $\mathcal{C}_3 + \mathcal{C}_4$ \vspace{-0.05in}
\item[$D_2$:] Site 1 has $\frac{1}{2}\mathcal{C}_1 + \mathcal{C}_2+ \frac{1}{2}\mathcal{C}_3$, and Site 2 has $ \frac{1}{2}\mathcal{C}_1 + \frac{1}{2}\mathcal{C}_3 + \mathcal{C}_4$\vspace{-0.05in}
\item[$D_3$:] Site 1 has a randomly selected half of the data and Site 2 the rest
\end{itemize}  
where $\frac{1}{2} \mathcal{C}_1$ means that a distributed site contains half of $\mathcal{C}_1$, 
and so on. 4000 data points are generated for training, and 1000 for test. The data compression ratio is 4:1 for K-means 
clustering and approximately 4:1 for rpTrees. 
\begin{figure}[h]
\centering
\begin{center}
\hspace{0cm}
\includegraphics[scale=0.54,clip]{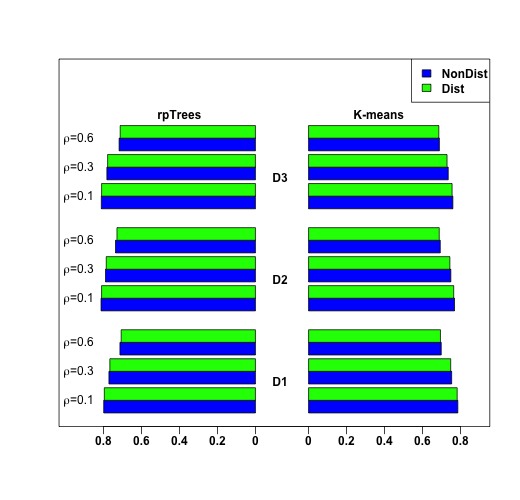}
\end{center}
\abovecaptionskip=-20pt
\caption{\it Test set accuracy of RF under simulation setting $D_1, D_2, D_3$, and different $\rho$'s. 
{\it `NonDist'} means all the data are in one place and no DML is applied, and {\it `Dist'} indicates our distributed algorithm. } 
\label{figure:RF}
\end{figure}
Figure~\ref{figure:logits} shows the classification accuracy of $L_1$ logistic 
regression for distributed and non-distributed data in different settings. It can be seen that, in all cases, the loss in classification 
accuracy due to distributed computing is small. 
\\
\\
The simulation settings for RF are the same as those for $L_1$ logistic regression. The number of trees in RF is set to be 100, 
and the {\it MTRY} parameter is taken from $\{\lfloor \sqrt{d} \rfloor, \lfloor 2\sqrt{d} \rfloor\}$ where $d$ is the data dimension. Figure~\ref{figure:RF} shows 
the classification accuracy of RF for distributed and non-distributed data under different settings. It can be seen that, in all cases, 
the loss in classification accuracy due to distributed computing is small. 
\subsection{Further simulations}
\label{section:furtherExp}
Apart from simulation settings $D_1, D_2$, and $D_3$, we also carry out simulations under two additional settings. 
The goal is to provide an evaluation to distributed settings that are potentially of interest. The first is 
on the effect of the number of participating distributed sites, and the second is when the class 
distribution are unbalanced at individual sites. This is done only for RF as it is much faster to evaluate and we would
expect the similar for $L_1$ logistic regression.
\begin{figure}[h]
\centering
\begin{center}
\hspace{0cm}
\includegraphics[scale=0.5,clip]{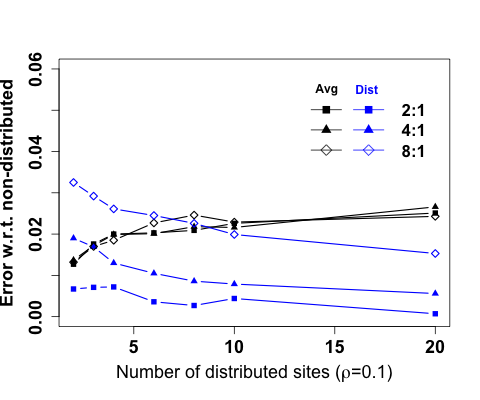}
\end{center}
\abovecaptionskip=-5pt
\caption{\it Test set error of distributed algorithm and local algorithm w.r.t. non-distributed (RF), 
at data compression ratio of 2:1, 4:1, and 8:1, respectively. } 
\label{figure:scen1N400r1}
\end{figure}
\\
\\
The first setting, labelled by $D_4$, is to explore the effect to distributed computing when the number of distributed 
sites is increasing. We compare the test set error of our proposed algorithm to the average of the test set error 
of individual sites, and the test set error when all the data are in one place. Comparison to the former would indicate 
if there is a performance gain by using our algorithm, while the later would serve as an upper bound on the performance 
of our algorithm. The number of distributed sites varies over $\{2,3,4,6,8,10,20\}$, and the data compression ratio is 
taken as 2:1, 4:1, and 8:1. Each site has data of size 1600. 
\\
\\
Figure~\ref{figure:scen1N400r1} shows the results for $\rho=0.1$. To avoid being overcrowded 
in the figure, we only show the difference in test set error w.r.t. that in non-distributed setting. The figure shows that as the number 
of participating sites increases, very soon our algorithm will outperform that using only local data (in terms of their average). 
Meanwhile, the performance gap between our algorithm and that in non-distributed setting quickly shrinks when the number of sites 
increases; for the same number of distributed sites, a lower data compression ratio leads to a smaller performance gap. 
Similar results can be seen for $\rho=0.3$ and $\rho=0.6$ in Figure~\ref{figure:scen1N400r3r6}. 
\begin{figure}[h]
\centering
\begin{center}
\hspace{0cm}
\includegraphics[scale=0.34,clip]{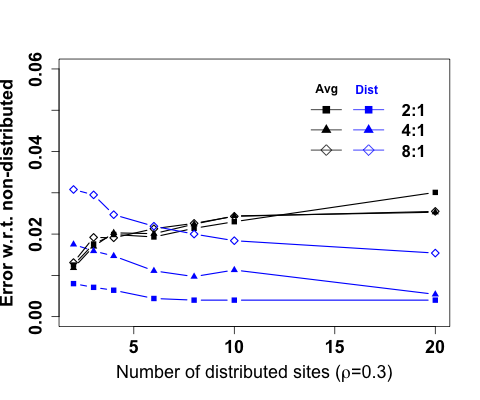}
\includegraphics[scale=0.34,clip]{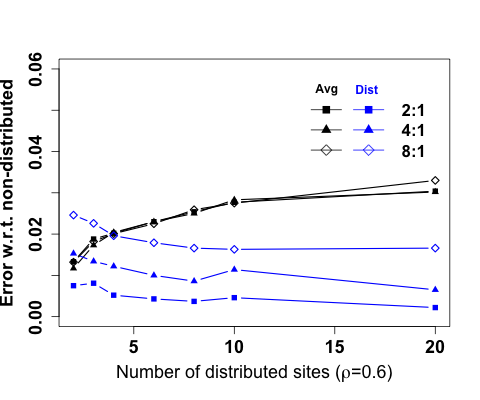}
\end{center}
\abovecaptionskip=-5pt
\caption{\it Test set error of distributed algorithm and local algorithm w.r.t. non-distributed (RF), 
at data compression ratio of 2:1, 4:1, and 8:1, respectively. } 
\label{figure:scen1N400r3r6}
\end{figure}
\\
Another simulation setting we consider is the case when there is class unbalance at individual sites; this is indicated by 
$D_5$. This occurs frequently in practice, as often any particular site has data only from certain sources thus are possibly 
to be unbalanced, while putting all the data in one place will likely reduce the extent of class unbalance. We produce training 
set that is unbalanced at individual sites as follows. First the data are split evenly at random and allocated to individual sites, 
then one class from each individual site is downsampled to create class unbalance. The selection of the downsampled class 
is carried out in a round-robin fashion, e.g., class $(i~mod~J)$ is downsampled at distributed site $i=1, 2, ..., J$.
\begin{figure}[h]
\centering
\begin{center}
\hspace{0cm}
\includegraphics[scale=0.42,clip]{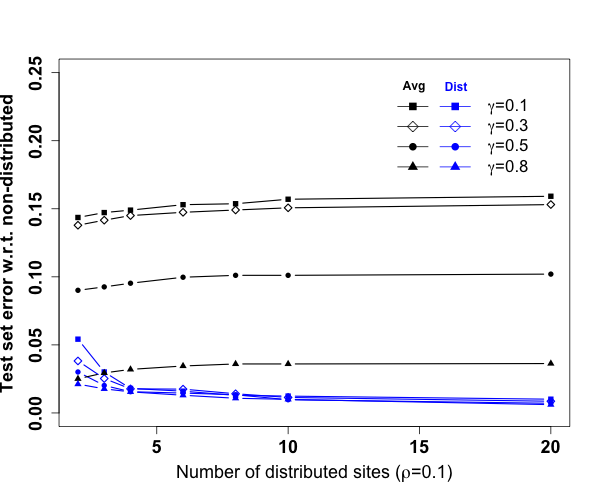}
\end{center}
\abovecaptionskip=-5pt
\caption{\it Test set error of distributed algorithm and average of local algorithms w.r.t. non-distributed algorithm (RF). 
$\gamma$ is the sampling ratio at individual distributed site to create class unbalance.  } 
\label{figure:scen2N400r1}
\end{figure}
\\
\\
Figure~\ref{figure:scen2N400r1} shows the test set error of our algorithm and the average error at individual sites when $\rho=0.1$. 
Similarly, we use the difference in test set error relative to that obtained in non-distributed setting. It can be seen that, by data sharing, 
our algorithm substantially reduces the test set error rate potentially caused by class unbalance. Similar results can be seen in 
Figure~\ref{figure:scen2N400r3r6} for $\rho=0.3$ and $\rho=0.6$. 
\begin{figure}[h]
\centering
\begin{center}
\hspace{0cm}
\includegraphics[scale=0.34,clip]{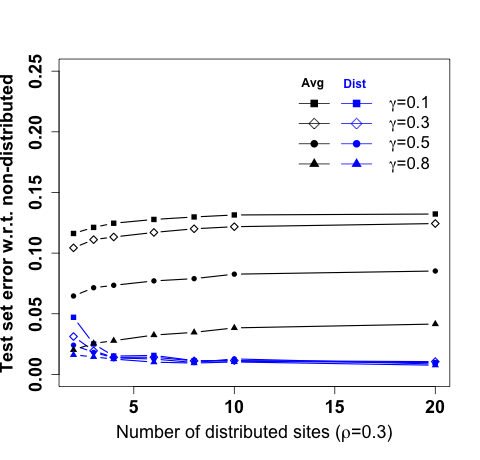}
\includegraphics[scale=0.34,clip]{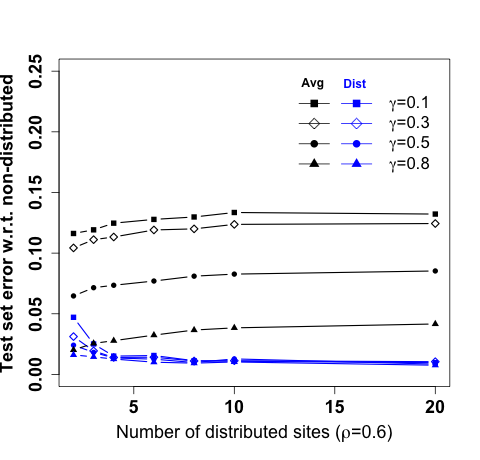}
\end{center}
\abovecaptionskip=-5pt
\caption{\it Test set error of distributed algorithm and local algorithm w.r.t. non-distributed (RF). $\gamma$ is 
the sampling ratio at individual distributed site to create class unbalance. } 
\label{figure:scen2N400r3r6}
\end{figure}
\subsection{UC Irvine data}
\label{section:expUCI}
To assess the performance on real data, we use benchmark data taken from the UC Irvine
Machine Learning Repository \cite{UCI}. Three datasets are used, including the Thyroid disease data, the Bank 
marketing data, and the Insurance benchmark data. Among the three datasets, one is from medical studies, and 
two from bank and insurance practice; such a choice reflects situations for which data are often inherently distributed 
but there is a resistance in data sharing by their owners. A summary of the datasets is given in Table~\ref{tbl:summaryUCI}.   

\begin{table}[h]
\begin{center}
\begin{tabular}{c|cc|c}
    \hline
\textbf{Data set}     & \textbf{\# Features}  &\textbf{\# instances}  & \textbf{Training}\\
    \hline
Thyroid disease   & 21       &7,200                &80\%\\
Bank marketing    &16		&45,211                &20\%\\
Insurance  &85	 &5,822		     &80\%\\
\hline
\end{tabular}
\end{center}
\caption{\it UC Irvine data sets used in our experiments. The last column shows the percentage of data used as 
the training sample (the rest for test).} \label{tbl:summaryUCI}
\end{table}
\begin{table}[h]
\begin{center}
\setlength{\extrarowheight}{2pt}
\begin{tabular}{c|c|c|c}
    \hline
\textbf{Data set}     & $\bm{D_1}$   	& $\bm{D_2}$  &$\bm{D_3}$ \\[2pt]
    \hline
Thyroid disease   & $\mathcal{C}_1+\mathcal{C}_2$    &$\mathcal{C}_1 +  \mathcal{C}_2+1/2\mathcal{C}_3$  &50\% \\[1pt]
\cline{2-4}
                             &$\mathcal{C}_3$ &1/2$\mathcal{C}_3$  & 50\%\\[2pt]
\hline
Bank marketing        &  $\mathcal{C}_1$     &70\%$\mathcal{C}_1$ +  30\%$\mathcal{C}_2$ &50\% \\[1pt]
\cline{2-4}
Insurance                 &$\mathcal{C}_2$       & 30\%$\mathcal{C}_1$ +  70\%$\mathcal{C}_2$ &50\%\\[2pt]
\hline
\end{tabular}
\end{center}
\caption{\it Simulation setting for UCI data sets. Under each of $D_1, D_2, D_3$, there are two rows corresponding to
Site 1 and Site 2, respectively.} \label{tbl:UCISetting}
\end{table}
\noindent
\\
Similar as simulations carried out for $L_1$ logistic regression and RF in Section~\ref{section:expL1LogitRF},
we also create three distributed settings, $D_{1-3}$. This is described in Table~\ref{tbl:UCISetting}. Experiments are carried 
out for $L_1$ logistic regression and RF, with K-means clustering and rpTrees as DML. Both 
the data compression ratio and the maximum size of tree leaf node are set to be 4. Table~\ref{tbl:accUCI} shows 
the test set accuracy for $L_1$ logistic regression and RF with K-means clustering and rpTrees as DML under 
distributed settings $D_{1-3}$, respectively. It can be seen that there is little or no loss in classification accuracy due 
to distributed computing in all cases.
\begin{table}[htp]
\begin{center}
\setlength{\extrarowheight}{2pt}
\resizebox{0.8\textwidth}{!}{%
\begin{tabular}{l|c|c|ccc}
    \hline
\textbf{Data set}              &\textbf{Alg}    & \textbf{NonDist}      	      &$\bm{D_1}$   &$\bm{D_2}$    &$\bm{D_3}$ \\[2pt]
    \hline
\multirow{4}{*}{\minitab[c]{Thyroid \\ disease} }   &\multirow{2}{*}{RF}       &\multirow{2}{*}{0.9928}                   &0.9860        &0.9863&0.9860\\
							\cline{4-6}
                                         &               &                                                                    &0.9841 &0.9848&0.9836\\
                                         \cline{2-6}
                                         &\multirow{2}{*}{$L_1$ logit}     &\multirow{2}{*}{0.9511}     &0.9594        &0.9580 &0.9554\\
                                         \cline{4-6}
                                         &                                             &             &0.9601 &0.9624 &0.9652\\
\hline
\multirow{4}{*}{\minitab[c]{Bank \\ marketing}}   &\multirow{2}{*}{RF}       &\multirow{2}{*}{0.8955}                   &0.8898        &0.8900  &0.8900\\
							\cline{4-6}
                                         &               &                                                                    &0.8974 &0.8938 &0.8887\\
                                         \cline{2-6}
                                         &\multirow{2}{*}{$L_1$ logit}     &\multirow{2}{*}{0.8898}      &0.8912        &0.8902 &0.8915\\
                                         \cline{4-6}
                                         &                                             &             &0.8915 &0.8914 &0.8919\\
\hline
\multirow{4}{*}{\minitab[c]{Insurance \\ benchmark}}   &\multirow{2}{*}{RF}       &\multirow{2}{*}{0.9283}           &0.9374    &0.9380  &0.9392\\
							\cline{4-6}
                                         &               &                                                                    &0.9374 &0.9391 &0.9370\\
                                         \cline{2-6}
                                         &\multirow{2}{*}{$L_1$ logit}     &\multirow{2}{*}{0.9404}      &0.9346        &0.9367 &0.9359\\
                                         \cline{4-6}
                                         &                                             &             &0.9343 &0.9341 &0.9321\\
\hline
\end{tabular}}
\end{center}
\caption{\it Test set accuracy of UC Irvine data sets. Each data with a particular method, RF or $L_1$ logit, corresponds to two rows, 
with the top row and bottom row for K-means clustering and rpTrees as DML, respectively. } \label{tbl:accUCI}
\end{table}
\section{Conclusions}
\label{section:conclusions}
We have proposed a general framework to enable learning and inference over inherently distributed data. This framework 
is based on a class of distortion minimizing local (DML) transforms. Under such a framework, one only shares a small amount 
of representative data among distributed sites, thus eliminating the need of having to transmit big data. Computation can be 
done very efficiently via {\it parallel local computation}. DMLs, implemented by K-means 
clustering or rpTrees, have desirable properties. Both can be computed efficiently, i.e., linear (or with an additional $\log$ factor) 
complexity in terms of the number of data points, and can be carried out in parallel at local nodes. Moreover, both are {\it statistically} 
efficient in the sense a vanishing error (w.r.t. the Bayes error) when the size of the local signatures increases. Thus the accuracy 
obtained under our framework will be close to that in non-distributed setting. This is demonstrated 
by simulations on both synthetic and real data under several distributed settings. One additional virtue of our framework 
is that, as the representative data need not to be in their original form, data privacy can also be preserved. 
\\
\\
We explore three popular learning and inference tools, including linear regression, $L_1$ logistic regression, and RF, and 
the additional errors incurred on metrics of interest are all small. Thus our proposed framework is promising as a general 
learning framework for distributed data. As a computational framework over inherently distributed data, our approach 
effectively solves the problem of data sharing in a distributed environment without compromising privacy. Our framework is 
expected to be applicable to a general classes of tools in learning and inference with the continuity property. Methods 
developed under our framework will allow practitioners to use potentially much larger data than previously possible, or to attack 
problems previously not feasible, due to the lack of data because of challenges in big data transmission or privacy 
concerns in data sharing. 
\section*{Acknowledgements}
We thank the editors and reviewers for helpful comments and
suggestions. 
\section*{Appendix: Proof of some theorems in Section~\ref{section:dmlTheory}}
\subsection{DML by K-means clustering}
\begin{proof}[Proof of Theorem~\ref{dmlKmeans}]
By Chebychev's inequality, we have, for all $\epsilon>0$,
\begin{eqnarray*}
P\left( || T_k(X) - X || \geq \epsilon \right) &\leq& \mathbb{E}||X-q(X)||^{\alpha}/\epsilon^{\alpha} \\
&\lesssim& C(f,\alpha,d) \cdot k^{-\alpha/d}  / \epsilon^{\alpha}\\
&\rightarrow& 0 
\end{eqnarray*}
as $k \rightarrow \infty$ where $C(f,\alpha,d)$ is a constant depending only on the source density $f$, the data 
dimension $d$, and a given $\alpha>1$.
\end{proof}
\begin{proof}[Proof of Theorem~\ref{dmlD}]
We have, for all $\epsilon>0$,
\begin{eqnarray*}
&& P\left( || T(X) - X || \geq \epsilon, X \in \cup_{j=1}^J S_j \right) \\
&\leq& \sum_{j=1}^J P\left( || T(X) - X || \geq \epsilon,  X \in S_j \right) \\
&=& \sum_{j=1}^J P\left( || T(X) - X || \geq \epsilon ~\vert  ~X \in S_j \right) \cdot P\left( X \in S_j \right)\\
&=& \sum_{j=1}^J p_j \cdot P\left( || T^{(j)}(X) - X || \geq \epsilon  \right) \\
&\rightarrow& 0 
\end{eqnarray*}
as $min(k_1,...,k_J) \rightarrow \infty$ where $k_j$ indexes transforms at site $S_j, j=1, 2, ...,J$.
In the above, the last step applies Theorem~\ref{dmlKmeans} on each individual site. 
\end{proof}
\subsection{DML by space-partitioning trees}
\begin{proof}[Proof of Theorem~\ref{dmlKD-tree}]
The key of the proof is to show that, the leaf node containing $X$ has been cut `many' times in probability thus 
has a vanishing diameter. One strategy is to show that the resulting tree is {\it full} (that is, up to certain height, 
all nodes have both left and right child nodes) up to many, say $k$, levels. 
\\
\\
Following the proof idea of Theorem 20.2 
in \cite{DevroyeGL1996}, we define {\it `good cuts'} as those cuts within $\sqrt{1+\epsilon}/2$ of the true median. 
That is, assume a node has probability mass $p$, and the probability mass of its  left and right children are $p_L$ 
and $p_R$, respectively, then
\begin{equation*}
max(p_L, p_R) \leq \frac{\sqrt{1+\epsilon}}{2}p.
\end{equation*} 
We wish to show that, up to level $k$, the probability of a bad cut vanishes as $k$ grows at a certain rate. This 
is achieved by estimating the probability of a `bad' cut. Such a probability would be small, if a node has 
many data points. Indeed that would be exponentially small, by the convergence of the empirical distribution to the 
true distribution. We first estimate the probability of a bad cut, given a node of size $N$. According to the proof 
of Theorem 20.2 in \cite{DevroyeGL1996} (the convergence of empirical distribution to the true
distribution), we have 
\begin{eqnarray}
\label{eq:badCut}
P \left ( \mbox{The cut is bad}~ \vert~ N\right) 
\leq 2 \cdot \exp\left( -\frac{1}{2}N \left( \sqrt{1+\epsilon}-1\right)^2 \right).
\end{eqnarray} 
Now we can estimate the probability of all `good' cuts up to level $k$. This is done by considering the event that 
there is at least a bad cut up to level $k-1$. Let $G_i$ be the event of a good cut at the $i^{th}$ cut, $G^c_i$ indicate 
that the $i^{th}$ cut is `bad', and $N_i$ the size of node at the $i^{th}$ cut, $i=1,2,...,2^k-1$. Note that, $N_i$ may also 
denote the event that the $i^{th}$ cut is on a node with size $N_i$. This is for simplicity of description; the exact meaning 
of $N_i$ should be clear from the context.
\begin{eqnarray}
&& P \left ( \mbox{There is at least a bad cut up to level}~k-1 \right) \nonumber\\
&=& P \left (  \cup_{i=1}^{(2^{k-1}-1)} G_i^c \right)  \leq \sum_{i=1}^{2^{k-1}-1} P \left (   G_i^c\right)  \nonumber\\
&=& \sum_{i=1}^{2^{k-1}-1} \sum_{N_i} P \left (   G_i^c  ~ \vert~ N_i \right ) \cdot P(N_i) \nonumber\\
&\leq& 2 \sum_{i=1}^{2^{k-1}-1}  \sum_{N_i} \exp\left( -\frac{1}{2}N_i \left( \sqrt{1+\epsilon}-1\right)^2 \right) \cdot P(N_i)  \label{eq:badCut1}\\
&\leq& 2 \sum_{i=1}^{2^{k-1}-1}  \exp\left( -\frac{1}{2}K_n \left( \sqrt{1+\epsilon}-1\right)^2 \right) \sum_{N_i} P(N_i) \label{eq:badCut2}\\
&\leq& 2^{k} \exp\left( -\frac{1}{2}K_n \left( \sqrt{1+\epsilon}-1\right)^2 \right) \nonumber\\
&\leq& 2^{k} \cdot 2^{ -\frac{1}{2}K_n \left( \sqrt{1+\epsilon}-1\right)^2  }.  \nonumber
\end{eqnarray} 
In the above, \eqref{eq:badCut1} uses probability of a bad cut at a node with size $N_i$ (that is, \eqref{eq:badCut}), 
and \eqref{eq:badCut2} uses the fact that a node stops growing if it has less than $K_n$ points (thus all nodes except 
the last level, i.e., the $k^{th}$ level, have more than $K_n$ points). As $K_n=n/(2^k)$, and $\frac{n}{k2^k} \rightarrow \infty$, 
the above probability vanishes as $k\rightarrow \infty$. Thus, {\it the tree is full up to level $k$} (and so far all the cuts are `good') 
{\it in probability}. 
\\
\\
We remain to show that at level $k$, the diameter of all nodes vanishes in probability. For k-d trees, the node cuts are along each 
of the $d$ coordinates in a round-robin fashion. Thus any given coordinate will be cut once for every $d$ cuts. By the time we arrive 
at the $k^{th}$ level of the tree, the total number of cuts is
\begin{equation*}
C_k=1+2+2^2+...+2^{k-1}=2^k-1.
\end{equation*}
So each coordinate has been cut $\lfloor (2^k-1)/d \rfloor$ times. We wish to estimate how much each coordinate has shrunken. 
As decision trees are invariant w.r.t. monotone transforms on its coordinates, we assume that the marginal distribution along all 
coordinates are uniform over $[0,1]$. As all the cuts up to level $k$ are `good', each cut would shrink one coordinate of the 
associated tree node by at least $1-\frac{\sqrt{1+\epsilon}}{2}$. It follows that, at the $k^{th}$ level, all the tree nodes has every 
coordinate with a range of size at most
\begin{equation*}
\left(\frac{\sqrt{1+\epsilon}}{2} \right)^{\lfloor (2^k-1)/d \rfloor}.
\end{equation*} 
Thus, as $k \rightarrow \infty$, the diameter of a cell containing $X$ vanishes in probability. 
\end{proof}


\end{document}